\documentclass[twocolumn,10pt,twoside]{article}
\usepackage{flushend}
\usepackage[top=2.0cm,bottom=2.5cm,left=1.6cm,right=1.6cm,
            columnsep=0.65cm]{geometry}
\usepackage{booktabs}

\usepackage[T1]{fontenc}
\usepackage[utf8]{inputenc}
\usepackage{lmodern}
\usepackage{microtype}

\usepackage{fancyhdr}
\fancypagestyle{firstpage}{\fancyhf{}\fancyfoot[C]{\small\thepage}}

\usepackage{titlesec}
\titleformat{\section}{\normalfont\large\bfseries}{\thesection.}{0.55em}{}
\titleformat{\subsection}{\normalfont\normalsize\bfseries}{\thesubsection.}{0.5em}{}
\titlespacing*{\section}{0pt}{2.2ex plus .6ex minus .4ex}{1.0ex}
\titlespacing*{\subsection}{0pt}{1.4ex plus .4ex minus .2ex}{0.5ex}

\usepackage[colorlinks=true,
            linkcolor=blue!65!black,
            citecolor=teal!75!black,
            urlcolor=blue!65!black]{hyperref}

\usepackage{authblk}

\usepackage{abstract}

\usepackage{graphicx}
\usepackage{subcaption}
\usepackage{caption}
\usepackage{mathtools}
\usepackage{amsmath,amssymb,amsfonts,dsfont}
\allowdisplaybreaks[4]
\usepackage{amsthm}
\theoremstyle{plain}
\newtheorem{theorem}{Theorem}
\numberwithin{theorem}{section}

\newtheorem{lemma}[theorem]{Lemma}

\theoremstyle{remark}
\newtheorem{remark}[theorem]{Remark}

\usepackage{algorithm}
\usepackage{algpseudocode}

\usepackage{natbib}
\usepackage{xcolor}

\newcommand{\Central}{\textbf{Central Server}}
\newcommand{\Choosen}{\textbf{Chosen Worker Node}}

\newcommand{\ta}{\tilde{\alpha}}
\newcommand{\gp}{g'}
\newcommand{\tx}{\tilde{x}}
\newcommand{\xS}{x_*}
\newcommand{\ba}{\bar{A}}
\newcommand{\bm}{\bar{\mu}}
\newcommand{\bs}{\bar{\sigma}}
\newcommand{\cA}{\mathcal{A}}
\newcommand{\cF}{\mathcal{F}}

\newcommand{\cN}{\mathcal{N}}
\newcommand{\cX}{\mathcal{X}}
\newcommand{\bI}{\mathbb{I}}
\newcommand{\bE}{\mathbb{E}}
\newcommand{\bR}{\mathbb{R}}

\newcommand{\indc}[1]{\mathds{1}{\{#1\}}}
\newcommand{\sign}{\textnormal{sign}}
\newcommand{\supp}{\textnormal{supp}}
\newcommand{\tr}{\top}
\newcommand{\Var}{\textnormal{Var}}
\newcommand{\AGG}{\textnormal{AGG}}

\providecommand{\keywords}[1]{%
  \vspace{0.5em}\noindent\small\textbf{Key words:}\ #1\par\vspace{0.3em}}

\begin{document}

\title{\Large\bfseries Tight Convergence Rates for Online Distributed Linear Estimation with Adversarial Measurements}
\author[1]{Nibedita Roy}
\author[1]{Vishal Halder}
\author[1]{Gugan Thoppe}
\author[2]{Alexandre Reiffers-Masson}
\author[1]{Mihir Dhanakshirur}
\author[1]{Naman}
\author[2]{Alexandre Azor}
\affil[1]{Dept.\ of Computer Science and Automation, Indian Institute of Science, Bengaluru 560012, India}
\affil[2]{Department of Computer Science, IMT Atlantique, Plouzan\'e 29280, France}
\date{7 April 2026}
\maketitle
\thispagestyle{firstpage}

\begin{abstract}
We study mean estimation of a random vector $X$ in a distributed parameter-server-worker setup. Worker $i$ observes samples of $a_i^\top X$, where $a_i^\top$ is the $i$-th row of a known sensing matrix $A$. The key challenges are adversarial measurements and asynchrony: a fixed subset of workers may transmit corrupted measurements, and workers are activated asynchronously—only one is active at any time. In our previous work, we proposed a two-timescale $\ell_1$-minimization algorithm and established asymptotic recovery under a null-space-property-like condition on $A$. In this work, we establish tight non-asymptotic convergence rates under the same null-space-property-like condition. We also identify relaxed conditions on $A$ under which exact recovery may fail but recovery of a projected component of $\bE X$ remains possible. Overall, our results provide a unified finite-time characterization of robustness, identifiability, and statistical efficiency in distributed linear estimation with adversarial workers, with implications for network tomography and related distributed sensing problems.
\end{abstract}


\hrule height 0.4pt
\vspace{0.5em}

\section{Introduction}
\label{s:introduction}
Distributed learning and estimation form the backbone of modern large-scale applications—including federated optimization \cite{mcmahan2017communication}, sensor networks \cite{chong2003sensor}, and network monitoring \cite{vardi1996network}—where information is inherently decentralized across workers. Each agent observes only a fragment of an underlying signal, and a central server must aggregate these fragments to recover a global quantity of interest. The problem becomes substantially harder when an unknown subset of workers may be adversarial. Although adversary-resilient distributed optimization methods have been extensively studied, most existing approaches either assume homogeneous gradient structures, require synchrony, or guarantee convergence only to a neighborhood of the true solution under heterogeneous data. We review them in the remainder of the introduction.

Our prior work \cite{ganesh2023online} introduced a two-timescale algorithm for distributed linear estimation under fully asynchronous communication, adversarial corruption, and heterogeneous measurements. While it established asymptotic convergence guarantees, its non-asymptotic behavior and quantitative competitiveness with existing approaches remained unclear. The present work resolves these questions by deriving sharp convergence rates.

To formalize this contribution, we briefly outline the distributed estimation problem considered in this work. Our goal is to estimate the mean $\mu = \bE X$ of a random vector $X \in \bR^d$ in a parameter-server architecture. A known sensing matrix $A \in \bR^{N \times d}$ specifies how information about $X$ is distributed: worker $j$ observes samples of the scalar projection $a_j^\top X$, where $a_j^\top$ denotes the $j$-th row of $A$. Consequently, information about $X$ is fragmented across workers, and recovery of $\mu$ requires aggregating measurements from all workers. Communication with the server may occur either synchronously or asynchronously, and an unknown subset of workers may be adversarial and transmit arbitrary values. A formal description of this setup is provided in Section~\ref{s:setup.assumptions.main.result}.

This problem can be cast as distributed optimization:
\[
    \min_{x \in \bR^d} f(x), 
    \qquad 
    f(x) := \frac{1}{N} \sum_{j=1}^N f_j(x),
\]
where each $f_j$ encodes worker $j$’s sensing geometry and statistics. However, because the sensing vectors $a_j$ differ across workers, the resulting gradients are inherently heterogeneous, even in the absence of adversaries. Existing adversary-resilient methods to solve such a problem fall into three categories: \emph{data encoding} \citep{chen2018draco,data2019data,data2020data}, \emph{filtering} \citep{data2021byzantine,pillutla2022robust,damaskinos2018asynchronous,xie2020zeno++,fang2022aflguard,yang2021basgd}, and \emph{homogenization} \citep{ghosh2019robust,karimireddy2022byzantinerobust}. However, as we show, these approaches either are not inapplicable in our setting or suffer from significant limitations.

In data encoding schemes, workers compute stochastic estimates of carefully designed linear combinations of the local gradients $\nabla f_1(x), \ldots, \nabla f_N(x)$, introducing redundancy that enables the server to reconstruct the global gradient  $\nabla f(x) = \frac{1}{N} \sum_{j=1}^N \nabla f_j(x),$ and perform a descent step. However, such schemes require workers to access and process information corresponding to multiple gradient components. In our setting, each worker observes only its own scalar projection and therefore cannot compute coded combinations involving other workers’ measurements. Moreover, these methods typically rely on synchronous communication, which may be inefficient or infeasible when measurements arrive in real time or communication links are unreliable.

Filtering-based approaches can operate in both synchronous and asynchronous settings. In all such methods, each worker computes and transmits an estimate of its local gradient $\nabla f_j(x)$ to the server. In synchronous schemes, workers send these estimates simultaneously, and the server aggregates them using robust estimators such as the trimmed mean, coordinate-wise median, or geometric median, with the goal of suppressing adversarial outliers and approximating the global gradient.

In asynchronous settings, existing methods achieve adversarial robustness through three distinct strategies. 
(i) \emph{Server-side validation via trusted data:} the server maintains a private or trusted dataset and evaluates each received gradient against it, discarding updates deemed inconsistent \citep{xie2020zeno++,fang2022aflguard}. 
(ii) \emph{Model-based screening:} the server leverages structural properties of the objective—such as smoothness or descent conditions—to test whether an incoming update is plausible before incorporating it \citep{damaskinos2018asynchronous}. 
(iii) \emph{Delayed robust aggregation:} the server waits until a sufficiently large subset of worker updates has arrived and then applies a robust aggregation rule (e.g., trimmed mean or median), effectively recreating a synchronous filtering step within the asynchronous protocol \citep{yang2021basgd}.

Despite their appeal, these approaches face significant obstacles in our setting. Methods relying on private data implicitly assume that the server has reliable access to multiple measurement coordinates, which is unrealistic in our distributed sensing model. Moreover, for synchronous schemes and for the latter two asynchronous strategies, existing guarantees typically establish convergence only to a non-zero error neighborhood under gradient heterogeneity. For example, in \citep[Theorem~4]{damaskinos2018asynchronous}, the residual error depends on the smoothness constant of the objective, while in \citep[Theorem~1]{data2021byzantine} and \citep[Theorem~1]{yang2021basgd}, it depends on the heterogeneity gap. More fundamentally, \citep[Theorem~III]{karimireddy2022byzantinerobust} shows that such an error floor is unavoidable in heterogeneous settings. Consequently, exact recovery cannot, in general, be guaranteed by filtering-based methods in our setting.

The third category, \emph{homogenization} \citep{ghosh2019robust,karimireddy2022byzantinerobust}, has been proposed to address the heterogeneity-induced error floor. In \citep{ghosh2019robust}, workers are clustered based on similarity of their data distributions, and exact local solutions are recovered within each cluster. In contrast, \citep{karimireddy2022byzantinerobust} groups gradient estimates into randomized buckets, averages them within each bucket to reduce heterogeneity, and then applies a robust aggregation rule to the resulting homogenized gradients. The clustering-based approach is unsuitable in our setting, where the goal is to recover a single global estimate of $\bE X$ that incorporates information from all workers. While the randomized bucketing strategy guarantees vanishing error under heterogeneity, it relies on synchronous communication and is therefore incompatible with fully asynchronous operation.

In contrast to the above approaches, \cite{ganesh2023online} takes inspiration from \citep{fawzi2014secure} and introduces a two-timescale algorithm for online estimation of $\bE X$ in our setting. This method asymptotically is stochastic gradient descent applied to the non-smooth, non-strongly convex objective
\begin{equation}
\label{e:obj.fn}
    f(x) := N^{-1} \|Ax - \bE Y\|_1.
\end{equation}
A key feature of this formulation is that it accommodates heterogeneous measurements without sacrificing robustness. The algorithm operates under full asynchrony and requires only intermittent access to local scalar observations. Moreover, under a Nullspace-Property (NSP)-like condition on $A$ (see \eqref{e:robust.cond}), it guarantees exact recovery of $\bE X$ despite adversarial corruption. However, its non-asymptotic convergence rate remained unresolved. 

\textbf{Key contributions}: We derive the convergence rate our algorithm (see Section~\ref{s:setup.assumptions.main.result}) under both synchronous and asynchronous communication, and for both constant and diminishing stepsizes. The obtained rates are optimal within the class of first-order methods for non-smooth, non-strongly convex optimization---despite the presence of adversarial workers. A central step in the analysis establishes a sharp bound on the inner product between the estimation error and the average gradient formed from both honest and adversarial updates; see Lemma~\ref{lem:xn.diff.gp.eps.bound}. In Section~\ref{s:numerical.illustrations} we empirically our method against filtering- and homogenization-based approaches. The results show that while our algorithm is competitive in synchronous regimes, it substantially outperforms existing methods under asynchrony. Finally, in Section~\ref{sec:partial_identifiability}, we examine when the NSP-like recoverability condition can be ensured in practice and characterize structural regimes under which it holds. We then study the complementary case where this condition fails and identify the strongest guarantees that remain attainable. In particular, we show that while full recovery of $\bE X$ may be impossible, exact recovery of an identifiable projected component is still achievable under a weaker condition.

\section{Setup and Main Result}
\label{s:setup.assumptions.main.result}
We formalize the distributed estimation problem, recall our algorithm from \cite{ganesh2023online}, and present our main convergence result.

\textbf{Goal and Setup:}
Our goal is to estimate the mean of a random vector $X \in \bR^d$ in a distributed parameter-server architecture. The system consists of a central server and $N \ge d$ workers. A fixed but unknown subset $\cA \subseteq [N] := \{1,\ldots,N\}$, with $|\cA| \le m$, is adversarial (here $|\cdot|$ denotes cardinality). Each worker $j \in [N]$ has access to IID samples of the scalar random variable $Y(j) := a_j^\tr X,$ where $a_j^\tr$ is the $j$-th row of a known tall matrix $A \in \bR^{N \times d}$.

Communication between the workers and the server occurs in one of the following modes:
\begin{itemize}
    \item \textbf{Synchronous:} At each iteration $n \ge 0$, all workers simultaneously transmit one sample of their respective $Y(j)$ to the server.
    \item \textbf{Asynchronous:} At each iteration $n \ge 0$, a single worker $i_{n+1}$ is selected uniformly at random from $[N]$ and transmits one sample of $Y(i_{n+1})$ to the server.
\end{itemize}

Honest workers transmit genuine IID samples of their associated random variables, whereas adversarial workers may transmit arbitrary values. For each honest worker $j$, the transmitted samples are assumed to be independent of the past and mutually independent across workers.

\let\oldnl\nl
\newcommand{\nonl}{\renewcommand{\nl}{\let\nl\oldnl}}

\begin{algorithm}[t]
\caption{Online Algorithm to Estimate $\bE [X]$ \citep{ganesh2023online} }\label{alg:ganesh.algorithm}
\begin{algorithmic}[1] 
    \State \textbf{Input:} stepsize sequences $(\alpha_n)$ and $(\beta_n),$ projection set $\cX,$ and observation matrix $A$
    \State Initialize estimates of $\bE X$ and $\bE Y$ at the server to  $x_0 \in \cX$ and $y_0 = 0 \in \bR^N,$ respectively
    \For{each iteration $n \geq 0$} 

    \vspace{1ex}
    
    \Statex \hspace{1em}  \Central
        \State Uniformly sample agent index $i \equiv i_{n + 1}  \in [N]$ 
        
        \State Update $\bE X$ estimate using \label{st:x.update}
        \[
            x_{n + 1} =  \Pi_{\cX} \big(x_n + \alpha_n\ a_i\ \sign(y_n(i) - a_i^\tr x_n) \big)
        \]
         
        \Statex \hspace{1em} \Choosen\ $i\in [N]$
         \If{agent $i$ is honest}
             \State  Obtain a sample $Y_{n + 1}(i) \overset{\text{IID}}{\sim} Y(i)$ and send it \Statex \hspace{3em}to the central server
              \Else  \State Assign some (possibly malicious) value to \Statex \hspace{2.7em} $Y_{n + 1}(i)$ and send it to the central server
         \EndIf{} 

         \vspace{1ex}
         
         \Statex \hspace{1em} \Central
        \State  Update $\bE Y$ estimate: $\forall j \in [N],$ \label{st:y.update}
        \begin{multline*}
            y_{n + 1}(j) = y_n(j) + \beta_n\ \big[N Y_{n + 1}(i)\indc{j = i} -  y_n(j)\big],
        \end{multline*}
        \Statex \hspace{1em} where $\indc{}$ is the indicator

        \vspace{1ex}
    \EndFor
\end{algorithmic}
\end{algorithm}

\textbf{Algorithm:} The pseudo-code for \cite{ganesh2023online}'s algorithm for estimating $\bE X$ in the \textit{asynchronous} scenario is given in Algorithm~\ref{alg:ganesh.algorithm}. Each iteration of this algorithm has three phases. In the first phase, the server picks a worker\footnote{At several places, when the context is clear, we suppress  $i_{n + 1}$'s dependence on $n$ for notational simplicity.} $i \equiv i_{n + 1} \in [N]$ uniformly at random and updates the estimate of $\bE X$ using Step \ref{st:x.update}, which can be viewed as (sub)-gradient descent step with respect to $|y_n(i) - a_i^\tr x_n|.$ In this step, $\Pi_\cX$ is the Euclidean projection on to the set $\cX$ which is assumed to contain $\bE X.$ Also, for any $r \in \bR$, $\sign(r) = 1$ (resp. $-1$) if $r > 0$ (resp. $r < 0$) and $= 0$ when $r = 0.$ In the second phase, worker $i$ sets $Y_{n + 1}(i)$ to be an independently obtained sample of $Y(i),$ if it is \textit{honest}, and to some (potentially malicious) value, otherwise. Thereafter, worker $i$ communicates this value to the server. In the  final phase, the central server uses the value of $Y_{n + 1}(i)$ to update its estimate of $\bE Y$ as shown in Step \ref{st:y.update}. When all workers are honest, $y_n$ would 
converge to $\bE Y,$ which means $x_n$'s update rule at the server, asymptotically, can be seen as a stochastic gradient descent algorithm for minimizing \eqref{e:obj.fn}. 

The synchronous version of Algorithm~\ref{alg:ganesh.algorithm} is, for $n \geq 0,$
\begin{equation}
\label{e:sync.x.update}
    x_{n + 1} = \Pi_{\cX} \bigg(x_n + \alpha_n \sum_{j \in [N]} a_j \sign(y_n(j) - a_j^\tr x_n)\bigg) 
\end{equation}
and, for all $j \in [N],$
\begin{equation}
    \label{e:sync.y.update}
    y_{n + 1}(j) = y_n(j) + \beta_n[Y_{n + 1}(j) - y_n(j)].
\end{equation}

\textbf{Recap of results from \cite{ganesh2023online}}: That work used differential inclusion theory to show that $(x_n)$ obtained using Algorithm~\ref{alg:ganesh.algorithm} converges a.s. to $\bE [X]$ for $\cX = \bR^d$ (no projection). The result needed the following assumptions:
{
\renewcommand{\theenumi}{$\cA_\arabic{enumi}$}
\begin{enumerate}
    \item \label{a:target.vector} \textbf{\emph{Target vector}}: There exist $\bm, \bs > 0$ with $|\bE X(k)| \leq \bm$ and $\Var \big(X(k)\big) \leq \bs^2$ for all $k \in [d].$ 

    \vspace{1ex}
    
    \item \label{a:observation.matrix} \textbf{\emph{Observation matrix}}: $A$ satisfies
    \begin{equation}
    \label{e:robust.cond}
        \sum_{j \in S^c}|a^\tr_jx| >  \sum_{j \in S}|a^\tr_j x|
    \end{equation}
    for all $x\in\mathbb{R}^d\setminus{0}$ and all $S \subseteq [N]$ with $|S| = m.$ 

    \vspace{1ex}

    \item \label{a:ganesh.stepsize} \textbf{\emph{Stepsizes}}: $(\alpha_n)_{n \geq 0}$ and $(\beta_n)_{n \geq 0}$ are decreasing positive numbers such that $\max \{\alpha_0, \beta_0\} \leq 1,$ $\sum_{n \geq 0} \alpha_n = \sum_{n \geq 0}\beta_n = \infty,$ $\lim_{n \to \infty} \alpha_n/\beta_n = \lim_{n \to \infty}\beta_n = 0,$ and $\max\{\sum_{n \geq 0} \alpha_n^2, \sum_{n \geq 0} \beta_n^2,$ $ \sum_{n \geq 0} \alpha_n \gamma_n\}  < \infty,$ where $\gamma_n = \sqrt{\beta_n \ln (\sum_{k = 0}^n \beta_k)}.$
\end{enumerate}
}
An example for \ref{a:ganesh.stepsize} is $\alpha_n = (n + 1)^{-\alpha}$ (resp. $\beta_n = (n + 1)^{-\beta}$) with $\alpha \in (2/3, 1]$ (resp.  $\beta \in (1/2, 1] \cap (2(1 - \alpha), \alpha)$). 

\textbf{Main results}: We now state  our key result (Theorem~\ref{thm:main.rate.result}) on $(x_n)$'s convergence rate. We begin by introducing the required notation. For $0 \leq k \leq n$ and $k \leq t \leq n,$ let $\tx_k^n := \sum_{t = k}^n \ta_t x_t$ denote the tail-averaged iterate, where $\ta_t \equiv \ta_t^{k, n} := \alpha_t/\bigg[\sum_{\ell = k}^n \alpha_\ell\bigg].$ Also, let $\ba := \max\limits_{j \in [N]}\|a_j\|,$ and $\Delta$ and $C_N$ be as follows:
\begin{table}[h]
    \centering
    \vspace{-1ex}
    \begin{tabular}{|c|c|c|}
    \hline
    Notation & Asynchronous & Synchronous  \\
    \hline 
    & & \\[-2ex]
    $\Delta$ & $\sqrt{d \ba^2 (\bs^2 + \bm^2)}$ & $\sqrt{d \ba^2 \bs^2}$ \\
    $C_N$ & $\dfrac{2(N - m)}{\sqrt{N}}$ & $\dfrac{2(N - m)}{N}$ \\[2ex]
    %
    %
    \hline
    \end{tabular}
\end{table}

Further, let 
\[
    \eta := \min_{S : |S| = m}\ \min_{x \neq 0} \frac{1}{N \|x\|} \bigg[\sum_{j \in S^c}|a_j^\tr x| - \sum_{j \in S}|a_j^\tr x| \bigg],
\]
where $\|\cdot\|$ is the Euclidean norm.  \cite[Lemma 2]{ganesh2023online} shows that $\eta > 0;$ hence, $K:= \frac{2 m \bar{A}}{N \eta} + 1 < \infty.$ 

Our result additionally requires a structural condition on the projection set $\cX$.
{
\renewcommand{\theenumi}{$\cA_\arabic{enumi}$}
\begin{enumerate}
    \setcounter{enumi}{3}

    \item \label{a:projection.set} \textbf{\emph{Projection set}}: $\cX$ is a non-empty, compact, and convex set containing $\bE X.$
\end{enumerate}
}
Finally, for such a set $\cX,$ let $D_\cX := \max\limits_{x \in \cX} \|x - x_0\|$ and $E_0^y := \max\limits_{j \in \cA^c} \sqrt{\bE |y_0(j) - \bE Y(j)|^2},$ where $x_0 \in \cX$ and $y_0 \in \bR^N$ are initial estimates for $\bE X$ and $\bE Y,$ respectively. 

\begin{theorem}
\label{thm:main.rate.result}
Let $(x_n)$ and $(y_n)$ be either generated asynchronously (using Step~\ref{st:x.update} and Step~\ref{st:y.update} of Algorithm~\ref{alg:ganesh.algorithm}) or synchronously (using  \eqref{e:sync.x.update} and \eqref{e:sync.y.update}). Further, suppose \ref{a:target.vector}, \ref{a:observation.matrix}, and \ref{a:projection.set} hold.  Then, for  $r \in (0, 1)$ and $k = \lceil rn \rceil,$ where $\lceil \cdot \rceil$ is the ceil function, the following claims hold.
\begin{enumerate}
    \item (Constant $(\alpha_t), (\beta_t)$ stepsizes) \label{st:x.rate.alpha.beta.const} 
    Let $n \ge 3$ and define $\alpha_t = n^{-1/2}$, $\beta_t = (\ln n - 2\ln\ln n)/(2rn)$ for $0 \le t \le n$. Then,
    \begin{multline*}
        \bE f(\tx_k^n) \leq \bigg[\frac{2K D_{\cX}^2}{(1 - r)} + \frac{\ba^2}{2} + \frac{40r(N - m) E_0^y}{N (1 - r)}\bigg] \frac{1}{\sqrt{n}} \\
        + \bigg[\frac{C_N  \Delta}{\sqrt{2r}}\bigg] \sqrt{\frac{\ln n}{n}}.   
    \end{multline*}

    \item (Constant $(\alpha_t),$ decaying $(\beta_t)$ stepsizes) \label{st:x.rate.alpha.const.beta.decay}Let $n \geq 1,$ and $\alpha_t \equiv 1/\sqrt{n}, \beta_t = 1/(t + 1)$ for $0 \leq t \leq n.$ Then,
    \begin{multline*}
        \bE f(\tx_k^n) \leq \bigg[\frac{2K D_\cX^2}{1 - r} +  \frac{\ba^2}{2} \\
        + \frac{C_N \Delta}{(1 - r)} \left(\frac{1}{\sqrt{r}} + 2(1 - \sqrt{r}) \right) \bigg] \frac{1}{\sqrt{n}}.
    \end{multline*}

    \item (Decaying $(\alpha_t),$ $(\beta_t)$ stepsizes) \label{st:x.rate.alpha.beta.decay} Let $n \geq 2,$ and $\alpha_t = 1/\sqrt{t + 1}, \beta_t = 1/(t + 1)$ for $t \geq 0.$ Then,
    \[
        \bE f(\tx_k^n) \leq \bigg[\frac{4K D_\cX^2 + \left[2\ba^2 + 4C_N \Delta\right] \ln \left( \frac{2}{r} \right)}{1 - r} \bigg] \frac{1}{\sqrt{n}}.
       \]
\end{enumerate}
\end{theorem}

\begin{remark}
\emph{\textbf{(On the convergence rate of $(x_n)$)}}:
The expected objective value $\bE f(\tx_{\lceil rn \rceil}^n)$ decays at rate $O(\sqrt{\ln n}/\sqrt{n})$ or $O(1/\sqrt{n})$. These rates are order-optimal for first-order nonsmooth, non-strongly convex optimization, even without adversaries \cite[Section~2.2]{nemirovski2009robust}. If bounds on $K$, $D_\cX$, $\bar{A}$, $C_N$, and $\Delta$ are known, the constants can be  improved further \cite[(2.23), (2.25)]{nemirovski2009robust}.

Due to heterogeneity in the sensing vectors $(a_j)$, existing adversary-resilient methods (e.g., \cite{data2021byzantine,damaskinos2018asynchronous,yang2021basgd,karimireddy2022byzantinerobust}) typically guarantee convergence only to a non-zero error neighborhood. Exact convergence is shown in \cite[Theorem~IV]{karimireddy2022byzantinerobust}, but relies on homogenization and is restricted to synchronous settings. In contrast, we establish exact convergence of $\tx_k^n$ to $\bE X$ under both synchronous and asynchronous communication, with rates matching \cite[Theorem~IV]{karimireddy2022byzantinerobust} up to constants.
\end{remark}

\begin{remark} \emph{\textbf{(Synchronous vs. Asynchronous)}}:
The rates are order-wise identical in both settings; only the constants differ. For fixed $m$, $C_N = O(1)$ under synchronous updates but  $O(\sqrt{N})$ under asynchrony, reflecting the weaker averaging effect of single-worker updates.
\end{remark}

\begin{remark}
\label{rem:stepsize.choice}
\emph{\textbf{(Stepsize Choice)}}: 
For the $x_n$-updates, since $f$ is a nonsmooth convex function, the choices of $\alpha_t$ correspond to classical subgradient stepsizes that achieve optimal rates. Similarly, for the $y_n$-updates, which correspond to gradient descent on a smooth convex objective, the choices of $\beta_t$ are canonical.
\end{remark}

\section{Proof of Main Result}
We analyze the asynchronous and synchronous settings separately in Subsections~\ref{s:async} and \ref{s:sync}, respectively.

\subsection{Asynchronous Setup}
\label{s:async}

We first state convergence-rate bounds for general stepsizes. For any $u \in \bR^N,$ let  $\|u\|_{1, \cA^c} := \sum_{j \in \cA^c} |u(j)|.$ 
\begin{theorem}
\label{thm:async.conv.rate.generic}
    Let $(x_n)$ and $(y_n)$ be generated using Step~\ref{st:x.update} and Step~\ref{st:y.update} of Algorithm~\ref{alg:ganesh.algorithm}, respectively. Also, let $\Delta, \ba, E_0^y, \tx_{k}^n, K,$ and $D_\cX$ be as defined in Section~\ref{s:setup.assumptions.main.result}. Then, for any $j \in \cA^c,$ we have 
    \begin{multline}
    \label{e:y.conv.rate.generic}
        \bE|y_n(j) - \bE Y(j)|^2 \leq  (E_0^y)^2\ \prod_{\ell = 0}^{n - 1} (1 - \beta_\ell)^2   \\
        + N \Delta^2\ \sum_{t = 0}^{n - 1} \beta_t^2 \prod_{\ell = t + 1}^{n - 1} (1 - \beta_\ell)^2.
    \end{multline}
    Further, for $0 \leq k \leq n,$ we have
    \begin{multline}
    \label{e:x.conv.rate.generic}
         \bigg[\sum\limits_{t = k}^n \alpha_t \bigg] \bE f(\tilde{x}^n_k) \leq 2K D_\cX^2 \\ 
         + \sum\limits_{t = k}^n \left[\dfrac{2 \alpha_t}{N} \bE \|y_t - \bE Y\|_{1, \cA^c} + \dfrac{\alpha_t^2 \ba^2}{2}\right].
    \end{multline}
\end{theorem}

\begin{remark}[Comparison to existing literature]
    Our bound in \eqref{e:x.conv.rate.generic} parallels \cite[(2.18)]{nemirovski2009robust}, but differs in two key ways: it includes an additional term $\bE \|y_t - \bE Y\|_{1,\cA^c}$ because we also estimate $\bE Y$, and the factor $D_\cX^2$ is scaled by $K$, capturing the worst-case impact of adversaries. Notably, $K=1$ in the absence of adversaries.
\end{remark}

We defer the proof of this result to the latter half of this section. We now show how it implies Theorem~\ref{thm:main.rate.result}.

\begin{proof}[Proof of Statement~\ref{st:x.rate.alpha.beta.const} in Theorem~\ref{thm:main.rate.result}]
We first derive $(y_n)$ and $(x_n)$'s convergence rates assuming $\alpha_t \equiv \alpha$ and $\beta_t \equiv \beta,$ where $\alpha, \beta \in (0, 1)$ are arbitrary.

Let $j \in \cA^c.$ It follows from \eqref{e:y.conv.rate.generic} that, for any $n \geq 0,$
\begin{align*}
    \bE|y_n(j) & - \bE Y(j)|^2 \\
    \leq {} & (1 - \beta)^{2n}\ (E_0^y)^2  + N \beta^2 \Delta^2\  \sum_{t = 0}^{n - 1} (1 - \beta)^{2(n - 1 - t)} \\
    \leq {} & (1 - \beta)^{2n}\ (E_0^{y})^2 + N\beta \Delta^2,
\end{align*}
where the last inequality follows since $1/(2 - \beta) < 1$ which itself holds since $\beta < 1.$ Because $\sqrt{a + b} \leq \sqrt{a} + \sqrt{b}$ for any real numbers $a, b \geq 0,$ it then follows that 
\begin{align}
    \bE |y_n(j) - \bE Y(j)| 
    \leq {} & (1 - \beta)^n E_0^{y} +  \sqrt{N \beta} \Delta. \label{e:y.const.step.error.bd}
\end{align}

Next, we derive $(x_n)$'s convergence rate. We have
\begin{align*}
    \bE f(\tx_k^n) & \\
    \overset{\textnormal{(a)}}{\leq} {} & \frac{1}{(n - k + 1)\alpha} \bigg[2K
    D_\cX^2 +  \frac{(n - k + 1) \alpha^2 \ba^2}{2}  + \\&  \sum_{t = k}^n \Big[\frac{2 \alpha (N - m)}{N}\big[(1 - \beta)^t E_0^y + \sqrt{N \beta} \Delta\big]
    \Big] \bigg] \\
    \leq {} & \frac{2K
    D_\cX^2}{(n -  k + 1) \alpha} +  \frac{\alpha \ba^2}{2}  \nonumber \\
    & + \frac{2(N - m)}{N} \left[  \frac{E_0^y}{(n - k + 1)}  \frac{(1 - \beta)^k}{\beta} + \sqrt{N\beta} \Delta \right] \\
    \overset{\textnormal{(b)}}{\leq} {} & \frac{2K
    D_\cX^2}{(1 - r) n \alpha} +  \frac{\alpha \ba^2}{2}  \nonumber\\& + \frac{2(N - m)}{N}\left[\frac{E_0^y}{(1 - r)}  \frac{e^{-\beta rn}}{n\beta} + \sqrt{N \beta} \Delta \right]
\end{align*}
where (a) follows by using \eqref{e:y.const.step.error.bd} in \eqref{e:x.conv.rate.generic} along with the fact that $\|u\|_{1, \cA^c} \leq (N - m)\max_j |u(j)|,$ while (b) holds since $k = \lceil r n \rceil$ and $\beta \in (0, 1)$ imply  $k - 1 \leq rn \leq k$ and $(1 - \beta)^{k} \leq (1 - \beta)^{rn} \leq e^{-\beta r n}.$ 

For $n \geq 3,$ we have $\ln n - 2 \ln \ln n \leq \ln n$ and $10(\ln n - 2 \ln \ln n) \geq \ln n.$ Hence, for the $\beta$ specified in the statement, $  \sqrt{\beta} \leq \sqrt{\frac{\ln n}{2rn}}, \quad e^{-\beta r n} = \frac{\ln n}{\sqrt{n}},$ and $n\beta \geq \frac{\ln n}{20 r}.$ By substituting these inequalities and the value of $\alpha$ specified in the statement, we get the desired bound. 
\end{proof}

\begin{proof}[Proof of Statement~\ref{st:x.rate.alpha.const.beta.decay} in Theorem~\ref{thm:main.rate.result}]
Let $j \in \cA^c.$ By substituting $\beta_t = 1/(t + 1)$ in  \eqref{e:y.conv.rate.generic}, we get
\begin{equation}
\label{e:y.decay.step.error.bd}
    \bE |y_n(j) - \bE Y(j)| \leq \frac{\sqrt{N} \Delta}{\sqrt{n}}.
\end{equation}
Next, for $\alpha_t \equiv \alpha \in (0,1)$, substituting \eqref{e:y.decay.step.error.bd} into \eqref{e:x.conv.rate.generic} yields
\begin{align*}
    \bE f(\tx_k^n) \leq {} &  \frac{2K D_\cX^2 }{(n - k + 1)\alpha} + \frac{\alpha \ba^2}{2} \\
    {} & + \frac{2 (N- m) \Delta}{\sqrt{N} (n - k + 1) } \sum_{t = k}^n \frac{1}{\sqrt{t}} \\
    \leq {} & \frac{2K D_\cX^2 }{(n - k + 1)\alpha} + \frac{\alpha \ba^2}{2} \\
    {} & + \frac{2 (N- m) \Delta}{\sqrt{N} (n - k + 1)} \left[\frac{1}{\sqrt{k}} + 2(\sqrt{n} - \sqrt{k})\right].
\end{align*}
Now $k = \lceil r n \rceil$ implies $k - 1 \leq rn \leq k.$ Hence, for $n \geq 1,$
\begin{align*}
     \bE f(\tx_k^n) \leq {} & \frac{2K
     D_\cX^2}{(1 - r) n \alpha} +  \frac{\alpha \ba^2}{2} \\
     {} & + \frac{2 (N- m) \Delta}{\sqrt{N} (1 - r) n} \left[\frac{1}{\sqrt{rn}} + 2(1 - \sqrt{r}) \sqrt{n} \right] \\
     \leq {} & \frac{2K
     D_\cX^2}{(1 - r) n \alpha} +  \frac{\alpha \ba^2}{2} \\
     {} & + \frac{2 (N- m) \Delta}{\sqrt{N} (1 - r) \sqrt{n}} \left[\frac{1}{\sqrt{r}} + 2(1 - \sqrt{r}) \right],
\end{align*}
where the last relation follows since $n\sqrt{n} \geq \sqrt{n}.$

The claim now follows by substituting $\alpha = 1/\sqrt{n}.$ 
\end{proof}

\begin{proof}[Proof of Statement~\ref{st:x.rate.alpha.beta.decay} in Theorem~\ref{thm:main.rate.result}]
The bound on $(y_n)$ is as in   \eqref{e:y.decay.step.error.bd}. By using this bound in \eqref{e:x.conv.rate.generic}, we get
\begin{align}
\label{e:x.rate.yn.subs.alpha.t.generic}
    \Big[\sum_{t = k}^n \alpha_t \Big]\ \bE f(\tx_k^n) &\leq 2K D_\cX^2 + \frac{2 (N - m) \Delta}{ \sqrt{N}} \sum_{t = k}^n \frac{\alpha_t}{\sqrt{t}} 
    \nonumber
    \\& + \frac{\ba^2}{2} \sum_{t = k}^n \alpha_t^2.   
\end{align}
Now  $\alpha_t = \frac{1}{\sqrt{t + 1}}$ implies that $\max\bigg\{\sum_{t = k}^n \frac{\alpha_t}{\sqrt{t}}, \sum_{t = k}^n \alpha_t^2 \bigg\}$ $\leq {} 2 \ln \left(\frac{n + 1}{k}\right)$
for any $k$ such that $1 \leq k \leq n.$ 
Similarly, $\sum_{t = k}^n \alpha_t \geq 2\big[\sqrt{n + 2} -\sqrt{k + 1}\big].$
Using these inequalities in \eqref{e:x.rate.yn.subs.alpha.t.generic} gives
\begin{equation}
\label{e:rate.before.i.substitution}
    \bE f(\tx_k^n)
    \leq \frac{2K D_\cX^2 + \left[\frac{4(N - m) \Delta}{\sqrt{N}} + \ba^2  \right] \ln \left( \frac{n + 1}{k} \right)}{2(\sqrt{n + 2} - \sqrt{k + 1})}.
\end{equation}
Finally, for the case of $k = \lceil rn \rceil,$ we have $k \geq rn;$ hence, 
\[
    \ln \left(\frac{n + 1}{k} \right) \leq \ln\left(\frac{n + 1}{rn}\right) \leq \ln \left(\frac{2}{r}\right).
\]
Now, for $n \geq 2,$ we have $2(rn + 2) \leq (n + 2) (r + 1),$ 
which, along with the facts that $k \leq rn + 1$ and $r \leq 1,$ implies that $2(\sqrt{n + 2} - \sqrt{k + 1}) \geq 2\sqrt{n + 2} \bigg(1 - \sqrt{\frac{r + 1}{2}} \bigg) \geq \sqrt{n}(1 - r)/2$. The desired result is now easy to see. 
\end{proof}

We now formally prove Theorem~\ref{thm:async.conv.rate.generic}, starting with \eqref{e:y.conv.rate.generic}.

For $j \in \cA^c,$  $(y_n(j))$ is a weighted average solely of the $Y(j)$ samples and is unaffected by measurements from other workers, including adversarial ones. Its update is equivalent to stochastic gradient descent on the strongly convex objective $\phi(z) = \tfrac{1}{2}(z - \bE Y(j))^2,$ and hence standard convergence rate bounds apply as shown below. 

\begin{proof}[Proof of \eqref{e:y.conv.rate.generic} in Theorem~\ref{thm:async.conv.rate.generic}]
For $j \in [N]$ and $n \geq 0,$ Step~\ref{st:y.update}, Algorithm~\ref{alg:ganesh.algorithm}, shows that $ y_{n + 1}(j) - \bE Y(j) = (1 - \beta_n)\ [y_{n}(j) - \bE Y(j)] + \beta_n Z_{n + 1}(j),$ where $Z_{n + 1}(j) = N Y_{n + 1}(i_{n + 1}) \indc{i_{n + 1} = j} - \bE Y(j).$
For $j \in \cA^c,$ i.e., when node $j$ is honest, $Y_{n + 1}(i_{n + 1}) \sim Y(j)$ is generated with independent randomness  on the event $\{i_{n + 1} = j\}.$ Hence, $\bE Z_{n + 1}(j) = 0,$ which implies
\begin{align}
\label{e:y_n.j.decomposition}
    \bE |y_{n + 1}(j) - \bE Y(j)|^2 
    &= (1 - \beta_n)^2\ \bE |y_{n}(j) - \bE Y(j)|^2 \nonumber \\&  +\beta_n^2\ \bE Z^2_{n + 1}(j).
\end{align}

Moreover, for any $j \in \cA^c$ and $n \geq 0,$
\begin{align}
        \bE Z_{n + 1}^2(j) \overset{(a)}{=} {} & N \bE Y^2(j)- [\bE Y(j)]^2 \notag \\
        = {} & N \bE [Y(j) - \bE Y(j)]^2 + (N - 1) [\bE Y(j)]^2 \notag \\
        \overset{(b)}{=} {} & N \bE[a_j^\tr (X - \bE X)]^2 + (N - 1) [a_j^\tr \bE X]^2 \notag \\
        \overset{(c)}{\leq} {} & N \|a_j\|^2\ \big[\bE \|X - \bE X\|^2 + \|\bE X\|^2\big] \notag \\
        \overset{(d)}{\leq} {} & N d \max_j \|a_j\|^2 \max_k \big[\Var (X(k)) +  [\bE X(k)]^2\big] \notag \\
        \leq {} & Nd \ba^2 (\bs^2 + \bm^2) \notag \\
        = {} & \Delta^2, \label{e:Delta.Bd.Derivation}
\end{align}
where (a) holds because, on the event $\{i_{n + 1} = j\},$ $Y_{n + 1}(i_{n + 1}) \sim Y(j)$ is generated with independent randomness, (b) holds since $Y(j) = a_j^\tr X,$ (c) follows from the Cauchy-Schwarz inequality, while (d) holds since $\bE\|X - \bE X\|^2 \leq d \max_k \bE |X(k) - \bE X(k)|^2$ and $\|\bE X\|^2 \leq d \max_k |\bE X(k)|^2$. 

Using \eqref{e:Delta.Bd.Derivation} in \eqref{e:y_n.j.decomposition}, it follows that $ \bE |y_{n + 1}(j) - \bE Y(j)|^2 \leq (1 - \beta_n)^2\ \bE |y_{n}(j) - \bE Y(j)|^2 + \Delta^2 \beta_n^2.$ A simple induction then gives the desired result. 
\end{proof}

Our proof of $(x_n)$'s convergence rate in \eqref{e:x.conv.rate.generic} is non-trivial. We begin by rewriting Step~\ref{st:x.update} of Algorithm~\ref{alg:ganesh.algorithm} as 
\begin{equation}
\label{e:xn.split}
    x_{n + 1} = \Pi_{\cX} \left( x_n + \alpha_n[g'_n + \epsilon_n + M_{n + 1}]\right),
\end{equation}
where, for $n \geq 0,$
\begin{align}
    \gp_n \equiv {} & \gp(x_n, y_n)    %
    := \frac{1}{N} \bigg[\sum_{j \in \cA^c} \sign(\bE Y(j) - a_{j}^\tr x_n)\ a_j \notag \\
    {} & \hspace{5.5em} +  \sum_{j \in \cA} \sign(y_n(j) - a_j^\tr x_n)\ a_j\bigg] \label{e:gp_n.defn} \\
    \epsilon_n \equiv {} &  \epsilon(x_n, y_n)    
    := \frac{1}{N} \sum_{j \in \cA^c} a_j \bigg[\sign(y_n(j) - a_{j}^\tr x_n) \notag \\
    {} & \hspace{8em} - \sign(\bE Y(j) - a_{j}^\tr x_n)\bigg],  \label{e:eps_n.defn} 
\end{align}
and
\begin{equation}
\label{e:M_n.defn}
    M_{n + 1} := a_{i_{n + 1}}\ \sign(y_n(i_{n + 1}) - a_{i_{n + 1}}^\tr x_n) - g'_n - \epsilon_n.
\end{equation}
Separately, let
\begin{equation}
\label{e:g_n.defn}
    g_n \equiv g(x_n) := \frac{1}{N} \sum_{j = 1}^N \sign(\bE Y(j) - a_j^\tr x_n) a_j.
\end{equation}

An intuitive description of $g_n, \gp_n, \epsilon_n,$ and $M_{n+1}$ is as follows. The vector $-g_n$ is a true subgradient of $f$ at $x_n$, while $-\gp_n$ is its corrupted version due to the adversarial estimates $y_n(j),$ $j \in \cA.$ The term $\epsilon_n$ captures the error from imperfect estimation of $\bE Y(j)$ for $j \in \cA^c,$ and $M_{n+1}$ represents the noise from updating $x_n$ using a single randomly selected coordinate. Formally, $(M_n)$ is a martingale-difference sequence with respect to $(\cF_n)$, where $ \cF_n := \sigma(x_0, y_0, i_1, x_1, y_1, \ldots, i_n, x_n, y_n).$ 

The update rule in \eqref{e:xn.split} is similar to the one studied in \citep[Section 2.2]{nemirovski2009robust}, which establishes $O(1/\sqrt{n})$ rates for subgradient methods in non-smooth, non-strongly convex settings without adversaries. Those  results would apply if $\gp_n = g_n$ and $\epsilon_n = 0,$ so our main challenge is to handle $\gp_n + \epsilon_n:$ adversaries can corrupt $\gp_n,$ and the discontinuity of the sign function prevents $\epsilon_n$ from vanishing even as $y_n(j) \to \bE Y(j)$ for all $j \in \cA^c.$

The next lemma outlines how we address this challenge. 

\begin{lemma}
\label{lem:xn.diff.gp.eps.bound}
Let $x \in \bR^d$ and $y \in \bR^N$ be arbitrary. Then, for $K$ as defined in Section~\ref{s:setup.assumptions.main.result}, we have
\begin{align}
    (x - \bE X)^\tr g'(x, y)\leq {} & \frac{1}{K} (x - \bE X)^\tr g(x) \label{e:xn.diff.xS.gp.bound}\\
    \intertext{ and } 
    (x - \bE X)^\tr \epsilon(x, y) \leq {} & \frac{2}{N} \|y - \bE Y\|_{1, \cA^c}, \label{e:xn.diff.xS.eps.bound}
\end{align}
where $\|\cdot\|_{1, \cA^c}$ is as in Theorem~\ref{thm:async.conv.rate.generic}. 
\end{lemma}
\begin{remark}
To ensure global convergence, classical analyses, e.g.,  \citep{nemirovski2009robust}, rely on the negativity of $(x - \xS)^\tr g(x)$. In our adversarial setting, Lemma~\ref{lem:xn.diff.gp.eps.bound} shows an analogous property for $(x - \bE X)^\tr[\gp(x,y)+\epsilon(x,y)]$. In particular, \eqref{e:xn.diff.xS.gp.bound} shows that adversaries can degrade $(x - \bE X)^\tr g(x)$ by at most a factor $1/K$, with $K=1$ in the absence of adversaries. On the other hand, \eqref{e:xn.diff.xS.eps.bound} ensures  $\epsilon(x_n, y_n)$'s impact vanishes asymptotically.
\end{remark}

We prove Lemma~\ref{lem:xn.diff.gp.eps.bound} after showing a technical result.
\begin{lemma}
\label{lem:A.rel}
    Let $K$ be as defined in Section~\ref{s:setup.assumptions.main.result}. Then, 
    \[
        (K - 1)\sum_{j \in S^c} |a_j^\tr x| \geq (K + 1) \sum_{j \in S} |a_j^\tr x|
    \]
    for every $x \in \bR^d$ and every $S \subseteq [N]$ such that $|S| = m.$
\end{lemma}
\begin{proof} 
If $m = 0,$ then both sides are zero.

Let $m \ge 1.$ If $\sum_{j \in S} |a_j^\tr x| = 0,$ then the result is trivial. If not, then $\eta$'s definition shows $
\sum_{j \in S^c} |a_j^\tr x| - \sum_{j \in S} |a_j^\tr x| \ge N\eta \|x\|.$ Dividing by $\sum_{j \in S} |a_j^\tr x|$ and using $\sum_{j \in S} |a_j^\tr x| \le m\bar{A}\|x\|,$ we obtain
\[
\frac{\sum_{j \in S^c} |a_j^\tr x|}{\sum_{j \in S} |a_j^\tr x|} 
\ge 1 + \frac{N\eta}{m\bar{A}} 
= \frac{K+1}{K-1},
\]
which proves the claim.
\end{proof}

\begin{proof}[Proof of \eqref{e:xn.diff.xS.gp.bound} in Lemma~\ref{lem:xn.diff.gp.eps.bound}]
By using $(x - \bE X)$ in place of $x,$  Lemma~\ref{lem:A.rel} shows that $(K - 1) \sum_{j \in \cA^c} |(x - \bE X)^\tr a_j| \geq (K + 1) \sum_{j \in \cA} |(x - \bE X)^\tr a_j|.$ Since $|z| \geq z \sign(r)$ for any $z, r \in \bR,$ it then follows that  
\begin{eqnarray*}
    &&(K - 1) \sum_{j \in \cA^c} |(x - \bE X)^\tr a_j| \geq \sum_{j \in \cA} \bigg(|(x - \bE X)^\tr a_j|\\
    &&\;\;\;+ K \big[ (x - \bE X)^\tr a_j \big]\ \sign(y(j) - a_j^\tr x) \bigg).
\end{eqnarray*}
Now, since  $a_j^\tr \bE X = \bE Y(j)$ and $|z| = z \sign(z),$ we get
\begin{multline*}
    (K - 1) \sum_{j \in \cA^c} (x - \bE X)^\tr a_j\ \sign (\bE Y(j) - a_j^\tr x)  \leq \\  \sum_{j \in \cA} \! (x - \bE X)^\tr \! a_j \! \Big[\sign(\bE Y(j) - a_j^\tr x) - K \sign(y(j) - a_j^\tr x)\Big];
\end{multline*}
the inequality is reversed since we have also multiplied by $-1$ on both sides. Therefore,
$K (x - \bE X)^\tr \gp(x, y) \leq (x - \bE X)^\tr g(x).$ The verifies the desired result. 
\end{proof}

\begin{proof}[Proof of \eqref{e:xn.diff.xS.eps.bound} in Lemma~\ref{lem:xn.diff.gp.eps.bound}]
For any $r, r_1,$ and $r_2 \in \bR,$ it is straightforward to check that
\begin{multline}
\label{e:sign.diff.relation}
    |\sign(r_1 - r) - \sign(r_2 - r)| \\
    \leq 2 \times \indc{|r_1 - r_2| \geq |r - r_2|}.
\end{multline}
Hence, 
\begin{align*}
    (x -  \bE X)^\tr & \epsilon(x, y)  \\
    \overset{(a)}{=} {} &  \frac{1}{N} \sum_{j \in \cA^c} (x - \bE X)^\tr a_j\ \\ &\times\big[\sign(y(j) - a_j^\tr x) - \sign(\bE Y(j) - a_j^\tr x) \big] \\
    \overset{(b)}{\leq} {} & \frac{2}{N} \sum_{j \in \cA^c} |(x - \bE X)^\tr a_j| \\& \times \indc{|y(j) - \bE Y(j)| \geq |a_j^\tr x - \bE Y(j)|} \\
    \overset{(c)}{\leq} {} & 
    \frac{2}{N} \sum_{j \in \cA^c} |y(j) - \bE Y(j)|,
\end{align*}
where (a) holds from the definition of $\epsilon(x, y)$ in \eqref{e:eps_n.defn}, (b) is due to \eqref{e:sign.diff.relation}, while (c) is true because $|(x - \bE X)^\tr a_j| = |a_j^\tr x - \bE Y(j)| \leq |y(j) - \bE Y(j)|$ when $\indc{|y(j) - \bE Y(j)| \geq |a_j^\tr x - \bE Y(j)|} = 1.$ This proves the desired result. 
\end{proof}

We finally derive $(x_n)$'s convergence rate. 

\begin{proof}[Proof of \eqref{e:x.conv.rate.generic} in Theorem~\ref{thm:async.conv.rate.generic}]
We have
\begin{align*}
    \|x_{n + 1} - \bE X\|^2 
    \overset{(a)}{=} {} & \|\Pi_\cX(x_n + \alpha_n(g'_n + \epsilon_n + M_{n + 1})) -\\& \Pi_\cX(\bE X)\|^2 \\
    \overset{(b)}{\leq} {} & \|x_n + \alpha_n(g'_n + \epsilon_n + M_{n + 1}) - \bE X\|^2\\
    = {} & \|x_n - \bE X\|_2^2 + 2\alpha_n (x_n - \bE X)^\tr\\& (g'_n + \epsilon_n + M_{n + 1}) + \\& \alpha_n^2 \|\gp_n + \epsilon_n + M_{n + 1}\|^2 \\
    \overset{(c)}{\leq} {} & \|x_n - \bE X\|_2^2 + 2\alpha_n (x_n - \bE X)^\tr \\& (g'_n + \epsilon_n + M_{n + 1}) + \alpha_n^2 \ba^2,
\end{align*}
where (a) follows from \eqref{e:xn.split} and since $\bE X \in \cX$ which implies $\Pi_\cX(\bE X) = \bE X$, (b) holds since $\Pi_\cX$ is non-expansive, while (c)'s validity can be seen from \eqref{e:gp_n.defn}, \eqref{e:eps_n.defn}, and \eqref{e:M_n.defn}, which together imply $\gp_n + \epsilon_n + M_{n + 1} = a_{i_{n + 1}} \sign(y_n(i_{n + 1}) - a_{i_{n + 1}}^\tr x_n)$ and, hence, $\|\gp_n + \epsilon_n + M_{n + 1}\| \leq \|a_{i_{n + 1}}\| \leq \ba.$ 

Now, letting $E_n := \frac{1}{2} \bE \|x_n - \bE X\|_2^2$ and then using $\bE[M_{n + 1} |\cF_n] = 0,$ it follows that 
\begin{equation}
\label{e:async.En.Bd.}
  E_{n + 1} \leq E_n + \alpha_n \bE [(x_n - \bE X)^\tr (g'_n + \epsilon_n)] +  \frac{1}{2} \alpha_n^2 \ba^2.
\end{equation}
By using \eqref{e:xn.diff.xS.gp.bound} and \eqref{e:xn.diff.xS.eps.bound} in this relation, we then get 
\begin{equation}
\label{e:nemirovski.analogous.term}
\begin{split}
    E_{n + 1} &\leq E_n + \frac{\alpha_n}{K} \bE (x_n - \bE X)^\tr g_n \\&  +\frac{2 \alpha_n}{N} \bE \|y_n - \bE Y\|_{1, \cA^c} + \frac{\alpha_n^2 \bar{A}^2}{2}.
\end{split}
\end{equation}
This expression closely parallels \cite[(2.6)]{nemirovski2009robust}, with two key differences: the term $\bE (x_n - \bE X)^\tr g_n$ is scaled by $1/K$, reflecting adversarial impact, and an additional term $\bE \|y_n - \bE Y\|_{1,\cA^c}$ which accounts for the estimation error. Nevertheless, \citep{nemirovski2009robust}'s approach extends to \eqref{e:nemirovski.analogous.term}.

Since $-g_n$ is the convex function $f$'s sub-gradient at $x_n,$ 
\[
    \bE [(x_n - \bE X)^\tr (-g_n)] \geq \bE\big[ f(x_n) -
    f(\bE X)\big] = \bE f(x_n),
\]
where the last equality follows since $f(\bE X) = 0.$ Using this relation in \eqref{e:nemirovski.analogous.term}, we then get that 
\[
    \bE \alpha_n f(x_n) 
    \leq K(E_n - E_{n + 1}) + \frac{2 \alpha_n}{N}  \bE \|y_n - \bE Y\|_{1, \cA^c} + \frac{\alpha_n^2 \bar{A}^2}{2}.
\]
Since this relation is true for any $n \geq 0$ and $E_{n + 1} \geq 0,$ 
\begin{multline*}
    \bE \sum_{t = k}^n  \alpha_t f(x_t) \\ 
    \leq {} K E_k +  \sum_{t = k}^n  \left[\frac{2\alpha_t}{N} \bE \|y_t - \bE Y\|_{1, \cA^c} + \frac{\alpha_t^2 \bar{A}^2}{2} \right] .
\end{multline*}
The definition of $\ta_j$ from Section~\ref{s:setup.assumptions.main.result} now shows that 
\begin{multline*}
    \bigg[\sum_{t = k}^n \alpha_t \bigg]\ \bigg[\bE \sum_{t = k}^n  \ta_t f(x_k) \bigg] \\ \leq K E_k +  \sum_{t = k}^n \left[\frac{2\alpha_t}{N} \bE \|y_t - \bE Y\|_{1, \cA^c} + \frac{\alpha_t^2 \bar{A}^2}{2} \right].    
\end{multline*}
The desired bound in \eqref{e:x.conv.rate.generic} now follows after noting that $E_k \leq 2D_\cX^2$ and $f(\tx_k^n) \leq \sum_{t = k}^n \ta_t f(x_t).$  
\end{proof}

\subsection{Synchronous Setup}
\label{s:sync}

The synchronous case follows the same approach as the asynchronous one; we outline the key steps below. 

\begin{theorem}
\label{thm:sync.conv.rate.generic}
    Let $(x_n)$ and $(y_n)$ be generated using \eqref{e:sync.x.update} and \eqref{e:sync.y.update}, respectively. Also, let $\Delta, E_0^y,$ and $\tx_{k}^n$ be as defined in Section~\ref{s:setup.assumptions.main.result}. Then, for any $j \in \cA^c,$ we have 
    \begin{align}  \label{e:sync.y.conv.rate.generic}
        \bE|y_n(j) - \bE Y(j)|^2 & \leq  (E_0^y)^2\ \prod_{\ell = 0}^{n - 1} (1 - \beta_\ell)^2 \nonumber \\& + \Delta^2\ \sum_{t = 0}^{n - 1} \beta_t^2 \prod_{\ell = t + 1}^{n - 1} (1 - \beta_\ell)^2.       
    \end{align}
    Further, for any $0 \leq k \leq n,$ $\tx_k^n$ satisfies \eqref{e:x.conv.rate.generic}.
\end{theorem}
\begin{proof}
The $(y_n)$ bound follows directly from \eqref{e:sync.y.update} using $\bE |Y(j) - \bE Y(j)|^2 \leq \Delta^2$, as in \eqref{e:Delta.Bd.Derivation}.

For $(x_n)$, rewrite \eqref{e:sync.x.update} as $x_{n+1} = \Pi_{\cX}(x_n + \alpha_n[\gp_n + \epsilon_n]).$ By non-expansiveness of $\Pi_{\cX}$ and since  $\|\gp_n + \epsilon_n\| \leq \ba$, 
\begin{multline*}
\|x_{n+1} - \bE X\|^2 
\leq \|x_n - \bE X\|^2 \\
+ 2\alpha_n (x_n - \bE X)^\tr (\gp_n + \epsilon_n) + \alpha_n^2 \ba^2.
\end{multline*}
Taking expectations and defining $E_n = \tfrac{1}{2} \bE |x_n - \bE X|^2$ yields \eqref{e:async.En.Bd.}; the bound for $\tx_k^n$ follows as before.
\end{proof}

\begin{proof}[Proof of Theorem~\ref{thm:main.rate.result} (Synchronous)]
When $\beta_t \equiv \beta$ for some generic constant $\beta \in (0, 1),$ \eqref{e:sync.y.conv.rate.generic} and the arguments that lead up to \eqref{e:y.const.step.error.bd} show that $\bE|y_n(j) - \bE Y(j)| \leq E_0^y (1 - \beta)^n + \sqrt{\beta} \Delta.$ On the other hand, for $\beta_t = 1/(t + 1),$ the arguments that lead up to \eqref{e:y.decay.step.error.bd} show that $\bE|y_n(j) - \bE Y(j)| \leq \frac{\Delta}{\sqrt{n}}.$ Unlike in \eqref{e:y.decay.step.error.bd} and \eqref{e:y.const.step.error.bd}, note that the above bounds do not have $\sqrt{N}$ factor in front of $\Delta.$
Now, by arguing as in the proof of Theorem~\ref{thm:main.rate.result} in the asynchronous case, we get the desired bounds. 
\end{proof}

\begin{figure*}[!ht]
    \centering
    \includegraphics[width=\textwidth]{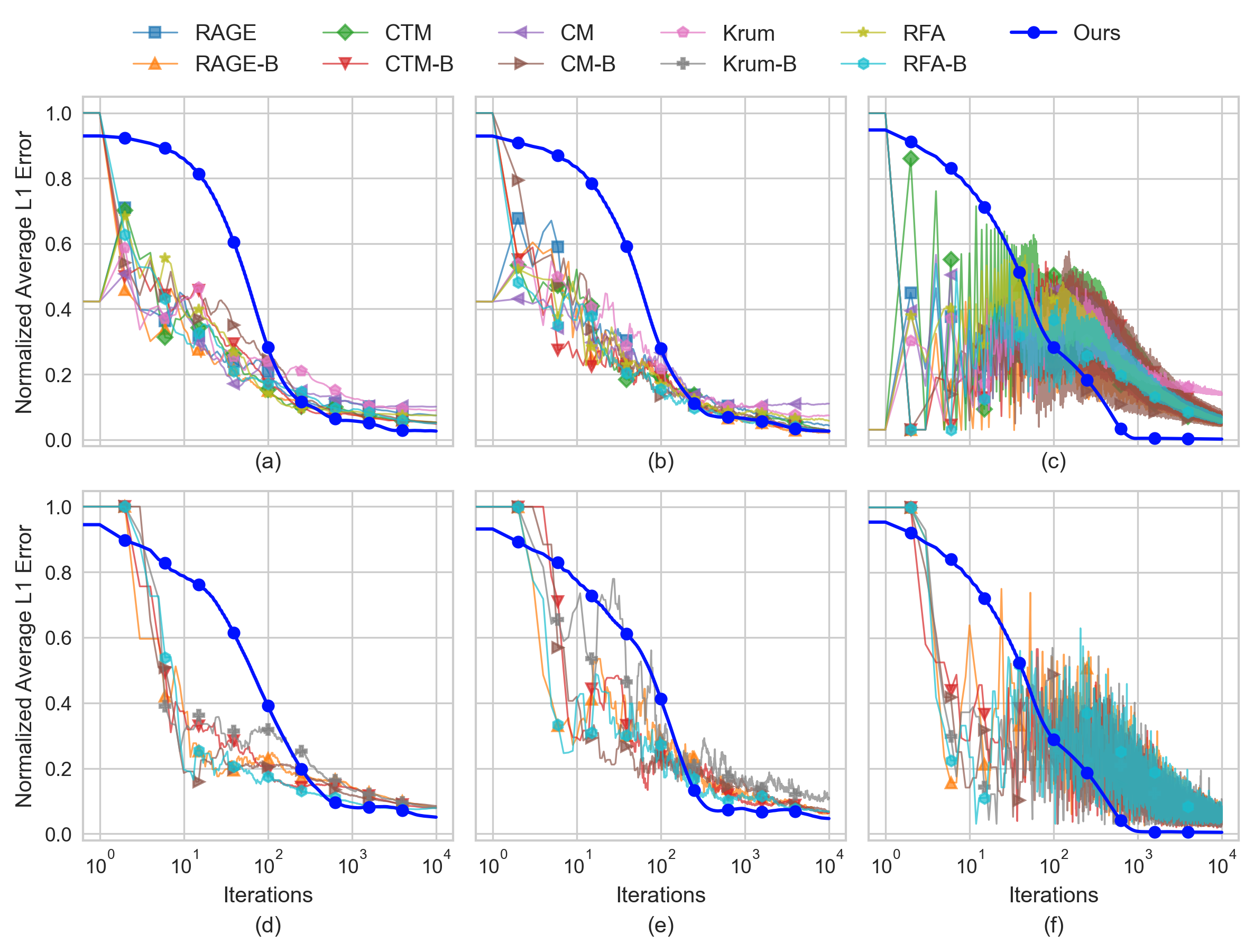}
    \caption{Comparison of Algorithm~\ref{alg:ganesh.algorithm}'s performance with that of existing state of the art robust-aggregator-based methods (see the legend to know the various methods we compare against). Subplots (a), (b), and (c) correspond to the synchronous case, while subplots (d), (e), and (f) correspond to the asynchronous case. In the synchronous-scenario plots, BKT stands for the bucketing variant, while it denotes the buffered variant in the asynchronous case. Subplots (a) and (d) correspond to the case where we update $x_n$ in Algorithm~\ref{alg:ganesh.algorithm} using the $1/\sqrt{n + 1}$ stepsize and the $x_n$ in other methods with the $1/(n + 1)^{0.9}$ stepsize. In subplots (b) and (e), we update $x_n$ for all methods using the $1/\sqrt{n + 1}$ stepsize. In subplots (c) and (f), we multiple $A$ by $10$ (to artificially add heterogeneity) and rerun the experiments with stepsizes as in subplots (b) and  (e). Note that the error in all the subplots is obtained by averaging over $10$ runs.}
    \label{fig:experiments}
\end{figure*}

\section{Numerical Illustrations}
\label{s:numerical.illustrations}

In this section, we compare our approach with existing adversary-resilient methods; the results are shown in Figure~\ref{fig:experiments}. We restrict attention to a single adversarial coordinate and the representative sensing matrix $A$ in Figure~\ref{fig:A.from.P.decomposition}. Our goal is a controlled comparison against standard robust-aggregation baselines rather than an exhaustive empirical study. In the synchronous case, we compare against KRUM \citep{blanchard2017machine}, Coordinate-wise Median (CM) \citep{yin2018byzantine}, Coordinate-wise Trimmed Mean (CTM) \citep{xie2019slsgd}, (approximate) geometric median or Robust Federated Aggregation (RFA) \citep{pillutla2022robust}, and Robust Accumulated Gradient Estimation (RAGE) \citep{data2021byzantine}, along with their bucketing variants \citep{karimireddy2022byzantinerobust}. In the asynchronous case, we compare against their buffered variants \citep{yang2021basgd}.

\begin{figure*}[htbp]
    \centering
    \begin{subfigure}[b]{0.3\linewidth}
        \centering
        \includegraphics[width=\linewidth]{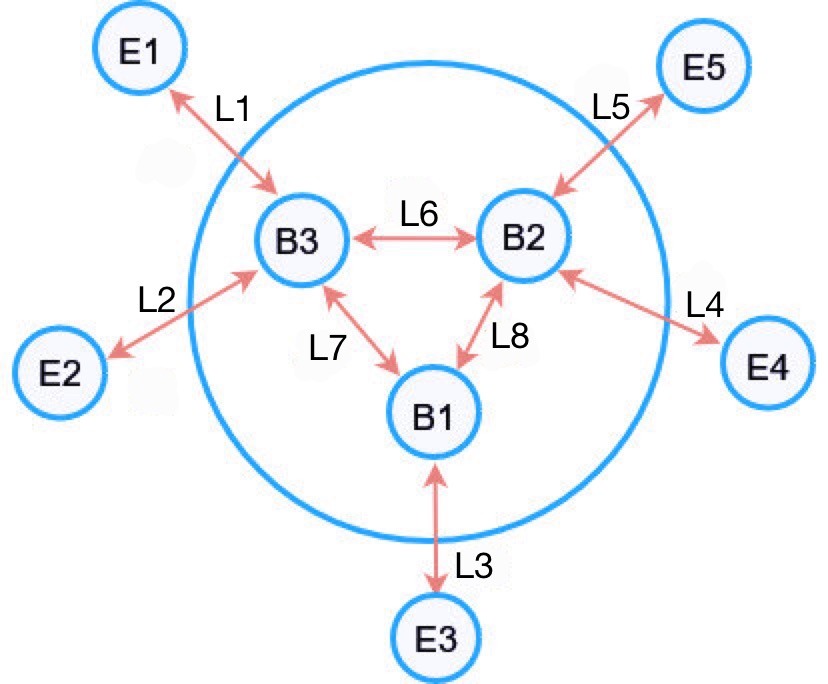}
        \caption{ A simple network example}
        \label{fig:network.setup}
    \end{subfigure}
    \hspace{0.02\linewidth} 
    \begin{subfigure}[b]{0.34\linewidth}
        \centering
        \(
        \renewcommand{\arraycolsep}{4pt}
        \renewcommand{\arraystretch}{1.2}
            P:=\begin{bmatrix}
                    0 & 0 & 1 & 1 & 0 & 0 & 0 & 1\\
                    1 & 0 & 0 & 1 & 0 & 1 & 0 & 0\\
                    0 & 1 & 1 & 0 & 0 & 0 & 1 & 0\\
                    0 & 1 & 1 & 0 & 0 & 1 & 0 & 1\\
                    0 & 0 & 1 & 0 & 1 & 1 & 1 & 0\\
                    1 & 0 & 0 & 0 & 1 & 0 & 1 & 1\\
                    0 & 0 & 0 & 1 & 1 & 1 & 1 & 1\\
                \end{bmatrix}
        \)
        \caption{Matrix $P$}
        \label{fig:path.link.matrix}
    \end{subfigure}
    \hspace{0.02\linewidth} 
    \begin{subfigure}[b]{0.28\linewidth}
        \centering
        \(
        \renewcommand{\arraycolsep}{4pt}
        \renewcommand{\arraystretch}{1.2}
            A:= \begin{bmatrix}
                    2 & 0 & 0 & 1\\
                    2 & 1 & 0 & 0\\
                    2 & 0 & 1 & 0\\
                    2 & 1 & 0 & 1\\
                    2 & 1 & 1 & 0\\
                    2 & 0 & 1 & 1\\
                    2 & 1 & 1 & 1
                \end{bmatrix}
        \)
        \caption{Matrix $A$}
        \label{fig:A.from.P.decomposition}
    \end{subfigure}
    \caption{Network setup, corresponding path-link matrix, and the matrix $A$ in the associated decomposition.}  \label{fig:network}
\end{figure*}

The implementation details are as follows. Recall that Algorithm~\ref{alg:ganesh.algorithm} solves the $\ell_1$-minimization problem in \eqref{e:obj.fn}. In contrast, we make the competing methods solve the $\ell_2$-minimization problem $\min \tilde{f}(x)$, where $\tilde{f}(x) = \frac{1}{N} \sum_{j = 1}^N (a_j^\tr x - \bE Y(j))^2$. This distinction is deliberate since robust-aggregation methods are typically developed for smooth, strongly convex objectives, where they are known to achieve faster convergence rates.

An abstract view of aggregation-based methods is as follows. In the synchronous setting, at each iteration \(n \ge 0\), every worker \(j\) sets \(Y_{n+1}(j)\) to a true sample of \(Y(j)\) if it is honest, and to an arbitrary value otherwise. The server then computes \(y_n(j)\), for all \(j \in [N]\), as in \eqref{e:sync.y.update}. It also forms the gradient estimates \(\hat{\nabla} f_j(x_n) = a_j (a_j^\top x_n - y_n(j))\) and the momentum terms \(\hat{m}_n(j) = (1 - \gamma_n)\hat{\nabla} f_j(x_n) + \gamma_n m_{n-1}(j)\). These are then aggregated to obtain $g_n = \AGG(\hat{m}_n(1), \ldots, \hat{m}_n(N)),$ where \(\AGG\) denotes the chosen aggregation rule. Finally, the server updates the solution estimate using $x_{n + 1} = x_n - \alpha_n g_n.$ 

For standard aggregators such as KRUM, CTM, CM, RFA, and RAGE, the aggregation rule \(\AGG\) is typically applied directly to \(\hat{m}_n(1), \ldots, \hat{m}_n(N)\). However, since the distributions of \(Y(j)\) and \(Y(j')\), as well as the values of \(a_j\) and \(a_{j'}\), differ for \(j \neq j'\), such direct application is known to introduce a non-zero bias. To mitigate this issue, \cite{karimireddy2021learning} proposed a \emph{bucketing} strategy. Under this approach, one first randomly permutes \(\hat{m}_n(1), \ldots, \hat{m}_n(N)\), and then partitions the permuted sequence into \(\lceil N/s \rceil\) buckets, each containing at most \(s\) elements. The values within each bucket are then averaged in a standard manner, and the chosen aggregation rule (e.g., CTM, CM, etc.) is subsequently applied to these bucket averages.

In the asynchronous case, at iteration \(n\), the server selects a worker \(i\) at random, queries it for its $Y(i)$ sample, and then updates \(y_n(j)\), for all \(j \in [N]\), as in Step~11 of Algorithm~\ref{alg:ganesh.algorithm}. It then computes \(\hat{\nabla} f_j(x_n)\) and \(\hat{m}_n(j)\) as in the synchronous case. However, we do not directly aggregate $\hat{m}_n(j)$ values since most gradient estimates are stale. We also cannot wait for all workers to report, as that would be inefficient. Instead, we build on \citep{yang2021basgd}, which proposes partitioning the \(N\) workers into \(\lceil N/s \rceil\) buffers, waiting until at least one worker in each buffer reports an estimate, averaging within each buffer, and then applying the chosen aggregation rule to these averages. Unlike bucketing, the worker-to-buffer assignments are mostly fixed and not randomly reshuffled.

In all our experiments, we have $N = 7$ workers and $m = 1,$ i.e., we have one adversary. Also, we set $\bE X = (5.47,\ 7.88,\ 11.51,\ 13.58)^\tr.$ At every iteration $n$, we set $Y_{n + 1}(j)$ to be $j$-th coorindate of the random vector $A (\bE X + \cN_4(0, \sigma^2 \bI_4)),$ where $A$ is as in Figure~\ref{fig:A.from.P.decomposition}, $\bI_4$ is the identity matrix, $\sigma = 100,$ and $\cN$ is the multi-variate Gaussian distribution. We always choose worker $7$ to be the  adversary in all our experiments, and set $s = 3$ for the bucketing and buffered variants. Finally, for subplots (a), (b), (d), and (e), we set the projection set $\cX$ to be the cartesian product of $[0, 30]$, while for subplots (c) and (f),  we set it to be the cartesian product of $[0, 300].$

We now discuss our stepsize choices, detailed in the caption of Figure~\ref{fig:experiments}. In all subplots and for all methods, we set $\beta_n = 1/(n + 1)$ for updating $y_n$, following Remark~\ref{rem:stepsize.choice}. For the $x_n$-updates, our method uses $\alpha_n = 1/\sqrt{n + 1}$, which is standard for nonsmooth convex optimization. For the competing methods, however, the appropriate stepsize is less clear: while $\alpha_n = c/(n+1)$ is optimal for smooth strongly convex problems, it requires careful tuning of $c$ based on $A$, and existing works instead favor $1/\sqrt{n}$ due to noise and robustness considerations. Motivated by this, we experiment with both choices, using a near-$1/n$ stepsize $\alpha_n = 1/(n + 1)^{0.9}$ in subfigures (a) and (d), and $\alpha_n = 1/\sqrt{n + 1}$ in the remaining subfigures. Furthermore, to understand the impact of heterogeneity, we scale the sensing matrix $A$ by a factor of $10$ and repeat the same setup as in subfigures (b) and (e). Finally, we set $\gamma_n = 1/(n + 1)^{0.9}$ for the synchronous approaches using bucketing, and $\gamma_n = 0$ elsewhere.

We now describe the adversarial strategy employed by Worker~7 at iteration $n$. We assume that the adversary has access to the values $m_n(1), \ldots, m_n(6)$. Based on these, it computes the coordinate-wise mean $\hat{\mu}_n$ and standard deviation $\hat{\sigma}_n$ of the six vectors. The adversary then selects $Y_n(7)$ so as to ensure that $m_n(7) = c_n a_7$, where $a_7^\tr$ denotes the seventh row of the matrix $A$ (see Figure~\ref{fig:A.from.P.decomposition}) and $c_n$ is chosen to minimize $\|c a_7 - (\hat{\mu}_n + \hat{\sigma}_n)\|_2$. This attack is inspired by the Baruch attack \citep{baruch2019little}, a commonly studied strategy in the literature. 

In all scenarios, Algorithm~\ref{alg:ganesh.algorithm} is competitive or outperforms existing methods, particularly in subplots (c) and (f), where heterogeneity is high.

\section{Beyond Full Recoverability}
\label{sec:partial_identifiability}
So far, we have studied the recovery of $\bE X$ under the strict, but necessary, $(\mathcal{A}_2)$ condition \citep{fawzi2014secure}. We now discuss two cases where it fails to hold. In one, we retain exact recovery under additional structure on $X$. In the other, a relaxed condition enables recovery of  $\bE X$'s projection.

\subsection{Recoverability with Additional Structure.}

When a sensing matrix $P \in \bR^{N \times p}$ does not satisfy $(\mathcal{A}_2)$, exact recovery of $\mu$ even from the deterministic signal $y = P\mu + e$ may be impossible  \citep{fawzi2014secure}. A natural approach to address this limitation is to impose additional \emph{structure} on $\mu$. Specifically, suppose $\mu = B\theta^\star$ for a known matrix $B \in \bR^{p \times d}$. Then $y = P B \theta^\star + e$, so the effective sensing matrix becomes $A := PB$. If $A$ satisfies $(\mathcal{A}_2)$, Algorithm~\ref{alg:ganesh.algorithm} can recover $\theta^\star$, and hence $\mu$, robustly, even in the presence of noise. Thus, recoverability can be restored under suitable structural assumptions.

We now illustrate the utility of this idea in network tomography \cite{vardi1996network, coates2002internet}. Consider the network in Figure~\ref{fig:network.setup} with path-link matrix $P \in \bR^{7\times 8}$ shown in Figure~\ref{fig:path.link.matrix}, where $P_{ij} = 1$ if link $j$ lies on path $i$ and $0$ otherwise. Let $X \in \bR^{8}$ denote the vector of random link delays and $Y = PX$ the vector of path delays. Since $P$ is wide, exact recovery of $\bE X$ is impossible, even without adversaries. Assume now that the five edge links $L1, \ldots, L5$ share the same mean delay, a standard assumption \citep{kinsho2019heterogeneous,kinsho2017heterogeneous}. Then $\bE X = B \bE \theta^\star$, where $\theta^\star \in \bR^{4}$ and
\[
    B^\tr =
    \renewcommand{\arraycolsep}{4pt}
    \renewcommand{\arraystretch}{1.2}
    \begin{bmatrix}
        1 & 1 & 1 & 1 & 1 & 0 & 0 & 0 \\
        0 & 0 & 0 & 0 & 0 & 1 & 0 & 0 \\
        0 & 0 & 0 & 0 & 0 & 0 & 1 & 0 \\
        0 & 0 & 0 & 0 & 0 & 0 & 0 & 1 
    \end{bmatrix}.
\]
It follows that  $A = PB$, where $A$ is as shown in Figure~\ref{fig:A.from.P.decomposition}. Since this $A$ satisfies $(\mathcal{A}_2)$, Algorithm~\ref{alg:ganesh.algorithm} can recover $\bE \theta^\star$ from noisy observations of $Y$, even when a subset of coordinates is adversarially corrupted. Consequently, $\bE X = B \bE \theta^\star$ can also be recovered. 

\subsection{Partial Recovery under Relaxed Conditions}
We now introduce a relaxed condition on $A$ that enables recovery of a \emph{projected component} of $\mu$. Since the same idea applies to both deterministic and noisy measurement models, we discuss only the deterministic case. 

($\mathcal{A}_2'$) \textbf{Partial recovery condition:}
There exist matrices $U \in \bR^{d\times r}$ and $V \in \bR^{d\times s}$ such that $\bR^d = \mathrm{span}(U)\oplus \mathrm{span}(V),$ and for every $K \subseteq [N]$ with $|K| \le q$ and every $g := A(U \alpha + V \beta)$ with $\alpha \neq 0$, we have
\begin{equation}
\label{eq:proj_nsp_pointwise}
\inf_{\beta\in\bR^s}
\left(
\sum_{i\in K^c} |g_i|
-
\sum_{i\in K} |g_i|
\right) > 0.
\end{equation}

\begin{remark}
Assumption $(\mathcal{A}_2')$ is strictly weaker than $(\mathcal{A}_2)$ as the following example illustrates. Let
\[
    A = 
    \begin{pmatrix}
    1 & 1 & 1 & 1 & 1\\
    0 & 0 & 0 & -1 & 1
    \end{pmatrix}^T,
    \quad
    U =
    \begin{pmatrix}
    1 & 0
    \end{pmatrix}^T,
    \quad
    V =
    \begin{pmatrix}
    0 & 1
    \end{pmatrix}^T.
\]
Then $(\mathcal{A}_2')$ holds (for $q=1$), but not $(\cA_2).$ 
\end{remark}

The next result shows that $(\mathcal{A}_2')$ suffices to recover the projection of the true signal onto $\mathrm{span}(U)$.

\begin{theorem}
\label{thm:alpha_recovery}
Let $A$ satisfy $(\mathcal{A}_2')$ for some $U$ and $V.$ Further, let $\mu = U \alpha^\star + V \beta^\star,$ $y = A\mu + e,$ and 
\[ 
    (\widehat \alpha,\widehat \beta)\in  \underset{\alpha\in\bR^r,\ \beta\in\bR^s}{\arg\min} \ \|A(U\alpha+V\beta)-y\|_1. 
\]
If $e$ is $q$-sparse, then $\widehat\alpha = \alpha^\star$.
\end{theorem}

\begin{proof}
We use proof by contradiction. Suppose $\widehat \alpha \neq \alpha^\star$ so that $\Delta \widehat \alpha := \widehat \alpha - \alpha^\star \neq 0.$ We show that this implies
\begin{equation}
\label{e:opt.hat.alpha.hat.beta}
    \|y - A\mu\|_1 < \|y - A(U \widehat \alpha + V \widehat \beta)\|_1,
\end{equation}
contradicting the optimality of $(\widehat \alpha, \widehat \beta).$

Let $K^\star:=\supp(e),$  $\Delta \widehat \beta = \widehat \beta - \beta^\star$, and $g^\star := U\Delta \widehat \alpha + V \Delta \widehat \beta$. Clearly, $\|y - A\mu\|_1 = \|e\|_1 = \sum_{i \in K^\star} |e_i|.$ Also, 
\begin{align*}
    \|y - & A(u \widehat \alpha + V \widehat \beta)\|_1 \\
    = {} & \|A g^\star -e\|_1 \\
    \overset{(a)}{=} {} & \sum_{i \in K^\star} |(Ag^\star - e)_i| + \sum_{i \in (K^\star)^c} |(Ag^\star)_i| \\
    \overset{(b)}{\geq} {} & \sum_{i \in K^\star} \Big[ |e_i| - |(A g^\star)_i| \Big]  + \sum_{i \in (K^\star)^c} |(Ag^\star)_i| \\
    \overset{(c)}{=} {} & \|e\|_1 - \sum_{i \in K^\star}|(A g^\star)_i|  + \sum_{i \in (K^\star)^c} |(Ag^\star)_i| \\
    \overset{(d)}{>} {} & \|e\|_1,
\end{align*}
where (a) and (c) follow from $K^\star$'s definition, (b) follows from traingle inequality, while (d) holds due to ($\mathcal{A}_2'$).

This proves \eqref{e:opt.hat.alpha.hat.beta}, which completes the proof. 
\end{proof}

\section{Conclusions and Future Directions}
We establish convergence rates for a two-timescale algorithm for adversary-resilient online estimation, offering a structure-driven alternative to aggregation methods. While our analysis focuses on distributed estimation, the underlying ideas extend more broadly. A key direction for future work is to generalize these techniques to machine learning and reinforcement learning settings.

\bibliographystyle{abbrvnat}
\bibliography{References}

\end{document}